\documentclass{article}

\usepackage[nonatbib,preprint]{neurips_2020}
\usepackage[square,numbers]{natbib}

\usepackage[utf8]{inputenc} 
\usepackage[T1]{fontenc}    
\usepackage{hyperref}       
\usepackage{url}            
\usepackage{booktabs}       
\usepackage{amsfonts}       
\usepackage{nicefrac}       
\usepackage{microtype}      

\title{A Variational View on Bootstrap Ensembles as Bayesian Inference}

\author{
  Dimitrios~Milios\\
  EURECOM\\
  Sophia Antipolis, France \\
  \texttt{dimitrios.milios@eurecom.fr} \\
  \And
  Pietro~Michiardi \\
  EURECOM \\
  Sophia Antipolis, France \\
  \texttt{pietro.michiardi@eurecom.fr} \\
  \And
  Maurizio~Filippone \\
  EURECOM \\
  Sophia Antipolis, France \\
  \texttt{maurizio.filippone@eurecom.fr} \\
}

\usepackage{xr}
\usepackage{graphicx}
\usepackage{xcolor}

\usepackage{algorithmic}
\usepackage{amsmath}
\usepackage{amssymb}
\usepackage{algorithm}
\usepackage{amsthm} 
\newtheorem{theorem}{Theorem}

\newtheorem{lemma}{Lemma}

\newtheorem{innercustomthm}{Theorem}
\newenvironment{customthm}[1]
	{\renewcommand\theinnercustomthm{#1}\innercustomthm}
	{\endinnercustomthm}

\usepackage{nicefrac}
\usepackage{mathtools}
\DeclareMathOperator{\Hessian}{Hess}

\usepackage{xspace}
\newcommand{\name}[1]{{\textsc{#1}}\xspace}

\newcommand{\mcmc}{\name{mcmc}}

\newcommand{\dnn}{\name{dnn}}
\newcommand{\dnns}{\textsc{dnn}s\xspace}
\newcommand{\relu}{{\textsc{r}}e\name{lu}}

\newcommand{\KL}{\name{kl}}
\newcommand{\kl}{\name{kl}}
\newcommand{\ML}{\name{ml}}
\newcommand{\MAP}{\name{map}}
\newcommand{\KDE}{\name{kde}}

\newcommand{\elbo}{\name{elbo}}

\newcommand{\mnll}{\name{mnll}}
\newcommand{\rmse}{\name{rmse}}

\newcommand{\mcd}{\name{mcd}}
\newcommand{\sghmc}{\name{sghmc}}
\newcommand{\gp}{\name{gp}}

\newcommand{\adagrad}{\name{adagrad}}
\newcommand{\lbfgs}{\name{l-bfgs}}

\newcommand{\norm}{\mathcal{N}}
\newcommand{\E}{\mathrm{E}}

\newcommand{\data}{\mathcal{D}}
\newcommand{\param}{\theta}
\newcommand{\wvec}{\mathbf{w}}

\begin{document}

\maketitle

\begin{abstract}
In this paper, we employ variational arguments to establish a connection between ensemble methods for Neural Networks and Bayesian inference. We consider an ensemble-based scheme where each model/particle corresponds to a perturbation of the data by means of parametric bootstrap and a perturbation of the prior. We derive conditions under which any optimization steps of the particles makes the associated distribution reduce its divergence to the posterior over model parameters. Such conditions do not require any particular form for the approximation and they are purely geometrical, giving insights on the behavior of the ensemble on a number of interesting models such as Neural Networks with ReLU activations. Experiments confirm that ensemble methods can be a valid alternative to approximate Bayesian inference; the theoretical developments in the paper seek to explain this behavior.
\end{abstract}

\section{Introduction}




Ensemble methods have a long history of successful use in machine learning to improve performance over individual models~\cite{DietterichMCS00}.
Recently, there has been a surge of interest in ensemble methods for Deep Neural Networks (\dnns)~\citep{Lakshminarayanan2017,Osband2018}. 
Contrary to early works on the topic, currents attempts use ensembles to characterize the uncertainty in predictions, rather than to improve performance~\citep{Hansen90}. 
Although the result of this practice is not conceptually too different from having a distribution of predictive models as in Bayesian inference,
to the best of our knowledge, ensemble-based methods lack the principled mathematical framework of Bayesian statistics.

Bayesian inference presents great computational challenges, as the evaluation of posterior distributions involves integrals that are often intractable~\citep{Bishop06}, and this is generally the case of models commonly employed in the literature, such as \dnns.
Bayesian approaches often resort to approximations either by means of sampling through Markov Chain Monte Carlo (\mcmc) \citep{Neal93,Neal96}, or by variational methods \citep{Jordan1999,Graves11}.  
The latter employ approximations of the posterior such that the Kullback-Leibler (\KL) divergence between the approximating and the target distributions is minimized.



The aim of this paper is to gain some insights on ensemble methods by viewing these through the lenses of variational inference.  
Ensemble-based methods rely on a wide range of practices, which are difficult to place under a single unified framework. 
In particular, ensemble methods include repeated training via random initialization~\citep{Szegedy2015,Sutskever2014,Krizhevsky2012}, random perturbation of the data~\citep{Lakshminarayanan2017}, or more recently, random perturbation of the prior~\citep{Osband2018,pearce2018_icml}.
In this work, we focus on parametric bootstrap \citep{Efron1993}, whereby 
many perturbed replicates of the original data are generated by introducing noise from a parametric distribution, and a new model is optimized for
each perturbed version of the original loss function.
Then, the ensemble of models, 
each member of which is represented by a ``particle'' in the parameter space, 
makes it possible to obtain a family of predictions on unseen data, which can be used to quantify uncertainty in a frequentist sense. 
In this paper we seek to investigate 
whether this ensemble has any connection with the ensemble of models obtained in Bayesian statistics when sampling parameters from their posterior distribution.

\begin{table*}[t]
\caption{
	Regression results: Average test \rmse and \mnll for ensembles of \dnns, Monte Carlo dropout (\mcd) and Stochastic-gradient Hamiltonian Monte Carlo (\sghmc). The \dnns used consist of 50 \relu nodes per layer.
} \label{tab:rmse_mnll}
\begin{center}
{\fontsize{6.7}{7.6} \selectfont
\begin{tabular}{l||cc|cc||cc|cc}
                                 &\multicolumn{4}{c||}{\textbf{TEST RMSE}}           
                                 &\multicolumn{4}{c}{\textbf{TEST MNLL}} \\
      &\multicolumn{2}{c|}{\textbf{Ensemble}}     &\textbf{MCD}     &\textbf{SGHMC}     
      &\multicolumn{2}{c|}{\textbf{Ensemble}}     &\textbf{MCD}     &\textbf{SGHMC}\\
\textbf{DATASET} &1-layer        &4-layer         &4-layer          &1-layer         
                 &1-layer        &4-layer         & 4-layer         &1-layer    \\ 
\hline \\
Boston (506)     &3.17$\pm$0.65   &3.17$\pm$0.60   &\textbf{2.93$\pm$0.27}    &3.55$\pm$0.57   
                 & 3.72$\pm$1.87  & 3.60$\pm$1.48  & \textbf{2.44$\pm$0.12}   & 3.40$\pm$0.87   \\
Concrete (1030)  &5.17$\pm$0.33   &5.47$\pm$0.75   &\textbf{4.74$\pm$0.34}    &6.17$\pm$0.40   
                 & 4.60$\pm$2.36  & 4.14$\pm$1.61  & \textbf{2.91$\pm$0.09}   & 5.20$\pm$1.06   \\
Energy (768)     &\textbf{0.45$\pm$0.04}   &0.65$\pm$0.38   &\textbf{0.45$\pm$0.04}    &0.46$\pm$0.04   
                 & 1.85$\pm$2.89  & 1.53$\pm$1.07  & \textbf{1.16$\pm$0.01}   & 1.19$\pm$1.04   \\
Kin8nm (8192)    &0.07$\pm$0.00   &\textbf{0.06$\pm$0.00}   &0.08$\pm$0.00    &0.08$\pm$0.00   
                 &-1.19$\pm$0.01  &\textbf{-1.30$\pm$0.01}  &-1.14$\pm$0.03   &-1.16$\pm$0.02   \\
Naval (11934)    &0.00$\pm$0.00   &0.00$\pm$0.00   &0.00$\pm$0.00    &0.00$\pm$0.00   
                 &\textbf{-5.61}$\pm$0.03  &-5.45$\pm$0.12  &-4.48$\pm$0.00   &-4.54$\pm$0.28   \\
Power (9568)     &4.12$\pm$0.11   &4.03$\pm$0.11   &\textbf{3.55$\pm$0.14}    &4.23$\pm$0.11   
                 & 2.88$\pm$0.03  & 2.98$\pm$0.06  & \textbf{2.63$\pm$0.03}   & 2.90$\pm$0.04   \\
Protein (45730)  &\textbf{1.89$\pm$0.03}   &\textbf{1.89$\pm$0.03}   &3.47$\pm$0.02    &1.96$\pm$0.03   
                 & \textbf{2.06$\pm$0.01}  &\textbf{2.06$\pm$0.01}   & 2.64$\pm$0.00   & 2.08$\pm$0.01   \\
Wine-red (1599)  &0.61$\pm$0.02   &0.61$\pm$0.02   &0.60$\pm$0.02    &0.62$\pm$0.02   
                 & \textbf{0.89$\pm$0.04}  & \textbf{0.87$\pm$0.04}  & \textbf{0.89$\pm$0.03}   & 0.94$\pm$0.04   \\
Yacht (308)      &0.75$\pm$0.23   &0.76$\pm$0.35   &1.49$\pm$0.28    &\textbf{0.57$\pm$0.20}   
                 & 1.12$\pm$0.34  & \textbf{0.96$\pm$0.33}  & 1.52$\pm$0.07   & 3.82$\pm$4.44   \\
\end{tabular}
}
\end{center}
\end{table*}

By introducing an appropriate prior-specific perturbation, in addition to the one due to the bootstrap,
it is possible to show that for linear regression with a Gaussian likelihood, the distribution of models obtained by bootstrap is {\em equivalent} to the one induced by the posterior distribution over model parameters (see, e.g., \citet{pearce2018_icml}).
This has sparked some interest in the literature of \dnns \citep{Osband2018,pearce2018_icml}, 
where quantification of uncertainty is highly desirable but the form of the posterior distribution is difficult to characterize. 
While scalable and flexible approximations for Bayesian \dnns exist \citep{Graves11,Rezende2015},
flexibility comes at the expense of increased complexity; simpler approximations cannot capture the intricacies of the actual posterior over model parameters \citep{Garipov2018}.
Nevertheless, an ensemble-based scheme may produce results that are competitive when compared to inference methods such as Monte Carlo dropout (\mcd) \cite{Gal16,gal2018w} 
and Stochastic-gradient Hamiltonian Monte Carlo (\sghmc) \cite{Chen14,BOHamiANN},
as seen in Table \ref{tab:rmse_mnll}.
The ensemble scheme that we consider is discussed in Section \ref{sec:methodology}, while a full account of the experimental setup can be found in Appendix \ref{sec:additional}.



In this paper, we aim to make a first step in the direction of connecting ensemble learning through bootstrap and Bayesian inference beyond linear regression.
We consider parametric bootstrap and the family of particles associated with the perturbed replicates of the data. 
We interpret the particles as samples from an unknown distribution approximating the posterior over model parameters, and we derive conditions under which any optimization steps of the particles improves the quality of the approximation to the posterior. 
Remarkably,
the conditions that we derive do not require assumptions on the distributional form assumed by the particles and they are purely {\em geometrical}, involving first and second derivative of the log-likelihood w.r.t. model parameters.
We make use of variational arguments to show that, in the linear regression case with a Gaussian likelihood, any optimization steps of the particles associated with a perturbed replicate of the data yields an improvement of the \KL divergence between the distribution of the particles and the posterior. 
Interestingly, \emph{the conditions that we derive suggest that this is also the case when the Hessian trace of the model function w.r.t.\ the parameters is zero almost everywhere.}
As a consequence, our result shows that applying parametric bootstrap on \dnns with \relu activations yields an optimization of the particles which 
does not degrade the quality of the approximation.


Turning Bayesian inference into optimization is very attractive, given that this could be massively parallelized while lifting the need to specify a flexible class of approximating distributions.
Experimental results, such as in Table \ref{tab:rmse_mnll}, confirm the potential of ensemble methods as an alternative to Bayesian approaches.
Crucially, and to the best of our knowledge, for the first time, we are able to gain insights as to why this is the case.


\section{Related work}



Ensemble methods are popular to boost the performance of \dnns \citep{Lee2015}.
Typically, diverse ensembles of networks are created by means of random initialization or randomly resampling dataset subsets.
Although ensembles have met significant empirical success \citep{Szegedy2015,Sutskever2014,Krizhevsky2012}, 
there is little discussion on potential connections with the Bayesian perspective.
One of the earliest attempts to bridge methodologically Bayesian inference and bootstrap can be attributed to \citet{Newton1994}.
More recently, \citet{Efron2012} proposed approximating Bayesian posteriors by reweighting parametric bootstrap replications.
In the area of \dnns, \citet{Heskes96} employed ensembles to evaluate uncertainty, where each member of the ensemble is trained on a non-parametric bootstrap replicate of the dataset.

A line of work worth mentioning in the area of particle-based approaches is Stein variational gradient descent \citep{SteinVI,SteinVI_2017}, which constitutes a deterministic algorithm to draw samples from a Bayesian posterior which relies on the kernelized Stein discrepancy \citep{Liu2016}.
Contrary to Stein-based approaches, particles in our work are considered to be random samples from a posterior approximation.

Bootstrapped ensembles have attracted attention in the recent years, as alternative to Bayesian inference.
In the work of \citet{Lakshminarayanan2017}, ensembles are used in combination with an adversarial training scheme.
The authors also suggest to use the negative log-likelihood as training criterion because it captures uncertainty.
In a few more recent works \citep{Osband2018,pearce2018_icml,pearce2018}, it is emphasized that bootstrapped ensembles 
of \dnns tend to converge to a solution that does not capture the uncertainty implied by the prior distribution.
\citet{Osband2018} address this issue by adding a different prior sample to a randomly initialized network and then optimize in order to obtain a member of the posterior ensemble.
In fact, our work is more related to the concept of \emph{anchored loss} introduced by \citet{pearce2018}, where each instance of an ensemble is ``anchored'' to a different sample of the prior.
In the case of a Gaussian posterior (i.e., Bayesian linear regression model) this process is proven to produce samples from the true posterior \citep{Osband2018,pearce2018_icml}.
To the best of our knowledge, there are no results in the literature regarding the case of Bayesian nonlinear regression models.

\section{Bootstrap ensembles}


In regression, observations $y$ are assumed to be a realization of a latent function $f(\mathbf{x}; \param)$ corrupted by a noise term $\epsilon$:
\begin{equation}
\label{eq:regression}
y = f(\mathbf{x}; \param) + \epsilon
\end{equation}
Given $n$ input-output training pairs $\data = \{(\mathbf{x}_i, y_i) \mid i=1\dots n\}$, the objective is to estimate $\param$. 
In Bayesian inference, this problem is formulated as a transformation of a prior belief $p(\param)$ into a posterior distribution by means of the likelihood function $p(\data|\param)$.
This is achieved by applying the Bayes rule:
$p(\param|\data) = \frac{p(\data|\param) \, p(\param)}{p(\data)}$,
where the evidence $p(\data)$ denotes the data probability when the model parameters are marginalized.
Except in cases where the prior and the likelihood function are conjugate, 
characterizing the posterior 
is analytically intractable \citep{Bishop06}.

Variational inference \citep{Jordan1999} recovers tractability by introducing an approximate distribution $q(\param)$ and minimizing the \KL divergence between $q(\param)$ and the posterior:
\begin{equation}
\kl[q||p] = \int_{\param} q(\param) \log \frac{q(\param)}{p(\param|\data)} d\param
\end{equation}
Equivalently, this problem is formulated as maximizing the Evidence Lower Bound (\elbo):
$\elbo[q] = \E_q [\log p(\data, \param)] - \E_q [\log q(\param)]$.
Contrary to the common practice of positing a parametric form for $q$ and optimizing the $\elbo$ with respect to its parameters~\citep{Graves11,Kingma14}, in this work, we consider $q$ to be an \emph{empirical} distribution, and we make no particular assumptions regarding its form.
In this case, a direct optimization of the \elbo presents numerical challenges due to the negative  entropy term $\E_q[\log q]$, as estimators for the empirical entropy are known to be biased \citep{Paninski2003}.
Our analysis of Section \ref{sec:analysis} allows us to reason about the gradient of the \KL divergence by gracefully avoiding the pitfalls of empirical entropies.



\paragraph{Ensemble of perturbed regression models}
\label{sec:methodology}
In a strictly frequentist setting, one can obtain a \emph{maximum a posteriori} estimate (\MAP) by maximizing: $\param_* = \arg\max_{\param} \left[ \log p(\data | \param) + \log p(\param)\right]$,
where $\log p(\param)$ is interpreted as a regularization term.
A typical frequentist strategy to generate statistics of point estimates is bootstrapping \citep{Efron1979}.
Data replicates are created by resampling the empirical distribution, and a different model is fitted to each replicate.
The ensemble of fitted models is used to calculate statistics of interest.
%
%
In parametric bootstrap \citep{Efron1993}, data replicates are created by sampling from a parametric distribution that is fitted to the data, typically by means of maximum likelihood (\ML).
In this work, we use the likelihood model as the resampling distribution, in order to reflect the assumptions of the Bayesian model.


\begin{minipage}{0.40\textwidth}
\begin{figure}[H]
	\includegraphics[width=\linewidth]{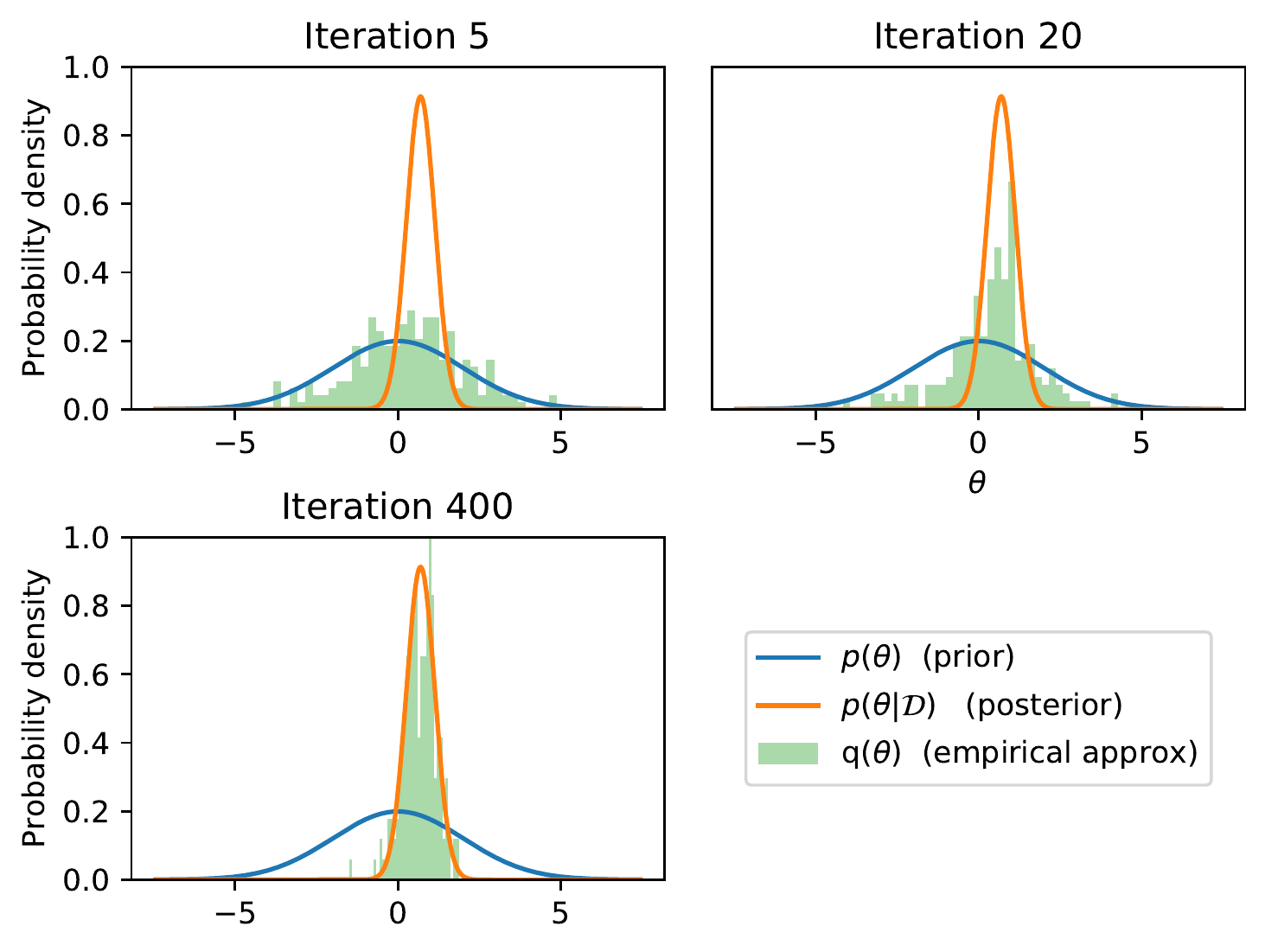}
	\vspace*{-20pt}
	\caption{Stages of optimizing particles for a univariate model.}
	\label{fig:univariate}
\end{figure}
\end{minipage}
\hspace{2pt}
\begin{minipage}{0.56\textwidth}
	\renewcommand{\algorithmicrequire}{\textbf{Input:}}
\renewcommand{\algorithmicensure}{\textbf{Output:}}
\begin{algorithm}[H]
	\caption{Ensemble of Perturbed Regression Models}
	\label{alg:the_algotithm}
	\begin{minipage}{0.48\textwidth}
	\end{minipage}
	{\footnotesize
	\begin{algorithmic}[1]
		\REQUIRE Joint log-likelihood $p(\data,\param)$, step size $h$
		\ENSURE A set of samples $\{\param^{(1)},\dots,\param^{(k)}\} \sim q$
		\FOR {$i \leftarrow 0$ {\bfseries to} $k$}
			\STATE Draw sample $\tilde{\data}$ from a likelihood-shaped \\ distribution with \ML parameter fitted on $\data$
			\STATE Draw sample $\tilde{\param}$ from the prior
			\STATE Initialize: $\param^{(i)} \leftarrow \tilde{\param}$
			\REPEAT 
				\STATE $\param^{(i)} \leftarrow \param^{(i)} + h \nabla \log \tilde{p}(\data, \param^{(i)})$
			\UNTIL {Convergence}
		\ENDFOR
	\end{algorithmic}
	}
\end{algorithm}

\end{minipage}

\vspace{5pt}
Consider a vectorization of model parameters $\param \in \mathbb{R}^m$.
For most of this work, we assume a Gaussian model for the noise, $\epsilon \sim \norm(0, \sigma^2)$, and a Gaussian prior over $\param$:
\begin{align}
\label{eq:normal_prior}
p(\data | \param) = \prod_{\mathbf{x}, y \in \data} \norm(y; f(\mathbf{x}; \param), \sigma^2)
\quad \text{~and~} \quad
p(\param) = \norm(0, \alpha^2 I_m)
\end{align}
For each likelihood component, we have exactly one observation $y$, which is also the \ML estimate for the mean parameter of a density with the same shape.
The label of each data-point will be resampled as follows:
\begin{equation}
	\tilde{y} \sim \norm(y, \sigma^2)
\end{equation}
where $\tilde{y}$ denotes a perturbed version of the originally observed label $y$.
We shall denote the entire dataset of perturbed labels as 
$\tilde{\data}$, such that $(\mathbf{x}, \tilde{y}) \in \tilde{\data}$.

In terms of a Bayesian treatment, variability is also explained by the prior distribution.
We shall capture this bevahior by introducing a perturbation on the prior, so that each perturbed model is attracted to a different prior sample.
Considering the prior of Equation \eqref{eq:normal_prior}, we create a perturbed version by resampling parameter components as follows:
\begin{equation}
p(\param;\tilde{\param}) = \norm(\tilde{\param}, \alpha^2 I_m),
\quad \text{ where } \quad \tilde{\param} \sim \norm(0, \alpha^2 I_m)
\end{equation}
The perturbed distribution $p(\param;\tilde{\param})$ depends on $\tilde{\param}$, which has been sampled from the original Gaussian prior.
The combined resampling will result in the following perturbed joint log-likelihood:
\begin{equation}
\label{eq:logp_peturbed}
\log \tilde{p}(\data, \param) = \log p(\tilde{\data}|\param) + \log p(\param;\tilde{\param})
\end{equation}

Along these lines, we propose the gradient ascent scheme of Algorithm \ref{alg:the_algotithm}, which operates on a set of particles $\{\param^{(1)},\dots,\param^{(k)}\}$, 
where $\param^{(i)} \in \mathbb{R}^m$ for $1 \le i \le k$.
We effectively maximize different realizations of $\log \tilde{p}(\data, \param)$, a process that can be trivially parallelized.
Each particle will be attracted to a different sample of the prior and a different perturbation of the data.

If $q$ is the empirical distribution of the particles, we hope that $q$ can serve as an approximation to the Bayesian posterior, as Algorithm \ref{alg:the_algotithm} approaches convergence.
The distribution $q$ is implicitly initialized to the prior, 
as we have $\theta^{(i)} \leftarrow \tilde{\theta} \sim p(\theta)$.
This is a sensible choice, as any samples away from the support of the prior are guaranteed to have very low 
probability under the posterior.
Notice that we make no other assumptions regarding the shape of $q$; we only know $q$ implicitly through its samples.
An illustration of how $q$ evolves over different stages of the optimization can be seen in Figure \ref{fig:univariate} for a univariate model with Gaussian posterior: 
we draw 200 particles from the prior (blue), which converge to an empirical approximation of the posterior (green histogram).
As a final remark, Algorithm \ref{alg:the_algotithm} can be seen as a special case of recent approaches in the literature \cite{Osband2018,pearce2018_icml}, which were used to quantify uncertainty for reinforcement learning applications.



\paragraph{Example -- Bayesian linear regression} We demonstrate the effect of Algorithm \ref{alg:the_algotithm} on a shallow model with trigonometric features (details in Appendix \ref{ssec:linear_example}).
The particles are represented by the predictive functions in Figure \ref{fig:toy_rff}.
The state of the particles at different optimization stages can be seen in the first four plots of the figure.
The particles approach the true posterior, whose samples can be seen in the rightmost plot of Figure \ref{fig:toy_rff}.
In this case, the ensemble-based strategy performs exact Bayesian inference, as we shall discuss in Section \ref{sec:lineal_models}.

\begin{figure*}
	\includegraphics[width=\linewidth]{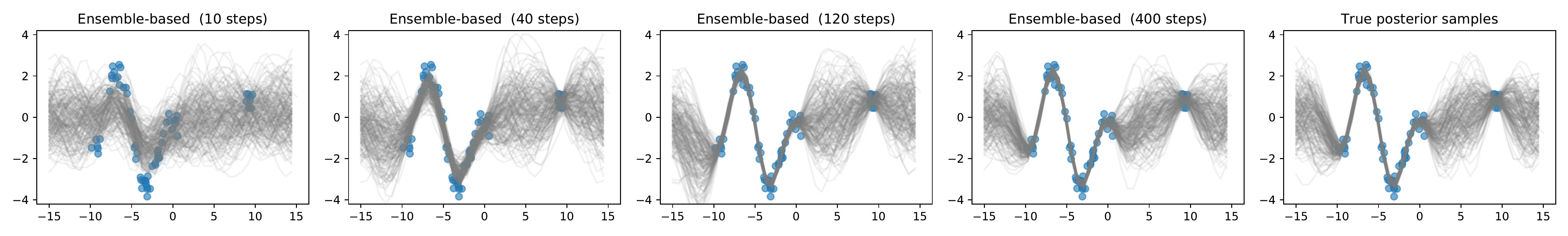}
	\vspace*{-15pt}
	\caption{Bayesian linear regression with trigonometric features -- State of 200 ensemble-based particles at different optimization stages}
	\label{fig:toy_rff}
\end{figure*}

\paragraph{Example -- Regression \relu network}
We next consider a 8-layer \dnn with 50 \relu nodes per layer; more details on the structure and the algorithms used can be found in Section \ref{ssec:relu_example} of the supplement.
Figure \ref{fig:demo_regression} shows the particles given by the ensemble scheme at different stages of Algorithm \ref{alg:the_algotithm}.
As a reference, we use samples of the Metropolis-Hastings algorithm.
As the optimization progresses, the distribution of predictive models improves until it reasonably approximates the \mcmc result.
This is the kind of behavior that we seek to explain in the section that follows.

\begin{figure*}
	\includegraphics[width=\linewidth]{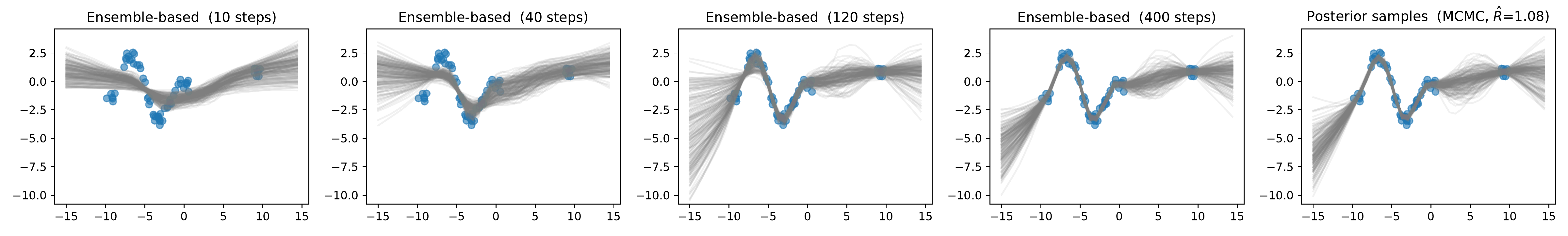}
	\vspace*{-15pt}
	\caption{Regression on a 8-layer \dnn with 50 \relu nodes -- State of 200 ensemble-based particles at different optimization stages.}
	\label{fig:demo_regression}
\end{figure*}




\section{Analysis of Perturbed Gradients}
\label{sec:analysis}

We investigate the effect of Algorithm \ref{alg:the_algotithm} on the implicit distribution $q$.
The linear case is treated first, for which there is an analytical solution.
We subsequently study the impact on nonlinear models.

\subsection{Bayesian Linear Models}
\label{sec:lineal_models}

For Bayesian linear models with $\param = \wvec \in \mathbb{R}^m$, the regression equation takes a simple form:
\begin{equation}
y = \wvec^\top \phi(\mathbf{x}) + \epsilon
\end{equation}
where $\epsilon \sim \mathcal{N}(0, \sigma^2)$ and $\phi(\mathbf{x}) \in \mathbb{R}^{D \times 1}$ is the projection of an input point onto a feature space of $D$ basis functions.
For a Gaussian prior over the weights $\wvec \sim \mathcal{N}(0, \alpha^2 I_m)$, the posterior is known to be Gaussian with mean and variance:
\begin{equation}
\mathrm{E}[\wvec] = \nicefrac{1}{\sigma^2} A^{-1} \Phi \mathbf{y} = \bar{\wvec} 
\quad \text{~and~} \quad 
\mathrm{Cov}[\wvec] = A^{-1}
\end{equation}
where $A = \nicefrac{1}{\sigma^2} \Phi \Phi^\top + \nicefrac{1}{\alpha^2} I_m$, the vector $\mathbf{y} \in \mathbb{R}^{n\times 1}$ contains all training outputs,
and $\Phi \in \mathbb{R}^{D \times n}$ is the design matrix of the training set in the feature space.
For a Gaussian posterior, the expected value is known to be the \MAP estimate.
We consider a perturbation of the labels according to the likelihood so that $\varepsilon \sim \mathcal{N}(0, \sigma^2)$. The prior is perturbed so that $\tilde{\wvec} \sim \mathcal{N}(0, \alpha^2 I_m)$.
Since the problem is convex, the effect on the \MAP estimate can be investigated directly; the \MAP solution will be:
\begin{equation}
\wvec_* = \nicefrac{1}{\sigma^2} A^{-1} \Phi (\mathbf{y} + \varepsilon)
+ \nicefrac{1}{\alpha^2} A^{-1} \tilde{\wvec}
\end{equation}
The distribution of $\wvec_*$ is also Gaussian with expectation $\E_{\varepsilon,\tilde{\wvec}}[\wvec_*] = \bar{\wvec}$ and covariance:
\begin{align}
\E_{\varepsilon,\tilde{\wvec}}[(\wvec_* - \bar{\wvec}) (\wvec_* - \bar{\wvec})^\top] = A^{-1}
\end{align}
which is equal to the covariance of the posterior over the weights
(see Section \ref{sec:suppl_linear} of the supplement).

\subsection{Effect on the gradient of KL-divergence}



Let $\kl[q||p]$ be the divergence between the approximating distribution $q(\param)$ and the posterior $p(\param|\data)$.
In Algorithm \ref{alg:the_algotithm}, the update on an individual particle in line 6 is described by the transformation:
\begin{equation}
\label{eq:tau}
	\tau(\param) = \param + h \nabla \log \tilde{p}(\data, \param)
\end{equation}
Assume that $\param \sim q$; then the transformation $\tau$ will induce a change in $q$ so that $\tau(\param) \sim q_{\tau}$.
It is desirable that the updated distribution $q_{\tau}$ is closer to the true posterior.
Thus, the derivative of the \KL divergence along the direction induced by $\tau$ has to be negative.
Then for a gradient step $h$ small enough, $\tau$ should decrease the \KL divergence and the following difference should be negative:
\begin{equation*}
	\delta_h = \kl[q_{\tau}||p] - \kl[q||p]
\end{equation*}
Because the approximating distribution $q$ is arbitrary, we shall take advantage of the fact that the \KL divergence remains invariant under parameter transformations \citep{Amari00}.
By applying the inverse transformation $\tau^{-1}$, we have the following equivalent difference:
\begin{equation}
\label{eq:kl_diff}
	\delta_h = \kl[q||p_{\tau^{-1}}] - \kl[q||p]
	       = \mathrm{E}_q [ \log p(\param|\data) - \log p_{\tau^{-1}}(\param|\data)]
\end{equation}
where $p_{\tau^{-1}}$ denotes the transformed posterior density by $\tau^{-1}$, which can be expanded as follows \citep{Bishop06}:
\begin{equation}
	p_{\tau^{-1}}(\param|\data) = p(\tau(\param)|\data) \det\{I_m + h \Hessian \log \tilde{p}(\data, \param)\}
\end{equation}
After substituting $p_{\tau^{-1}}$ in \eqref{eq:kl_diff}, we can calculate the directional derivative of the \KL along the direction of $\tau$ by considering the following limit:
\begin{equation}
\label{eq:kl_dir_deriv}
\lim_{h \to 0} \frac{\delta_h}{h}
= -\E_q[\nabla \log p^\top \nabla \log \tilde{p} + \mathrm{tr}{\{\Hessian \log \tilde{p}\}}] 
= \nabla_{\tau} \kl[q||p] 
\end{equation}
where:
\begin{equation*}
\label{eq:the_limits}
\begin{split}
\nabla \log p^\top \nabla \log \tilde{p} &= \lim_{h\to 0} \frac{\log p(\data, \param + h \nabla \log \tilde{p}) - \log p}{h}
\\
\mathrm{tr}{\{\Hessian \log \tilde{p}\}}&= \lim_{h\to 0} \frac{\log \det\{I_m + h \Hessian \log \tilde{p}\}}{h}
\end{split}
\end{equation*}
To keep notation concise,
we refer to the joint log-densities $\log p(\data, \param)$ and $\log \tilde{p}(\data, \param)$ simply as $\log p$ and $\log \tilde{p}$ correspondingly.
The first of the two limits above is the directional derivative towards the gradient $\nabla \log \tilde{p}$.
In a gradient ascent scheme, it is expected to have a positive value which gradually approaches zero over the course of optimization.


Ideally, the directional derivative in \eqref{eq:kl_dir_deriv} should stay negative (or zero) as $\log \tilde{p}(\data,\param)$ is maximized.
The conditions under which this is true are reflected in the following theorem:

\begin{theorem}
	\label{th:main_theorem}
	Let $\log \tilde{p}(\data,\param)$ be a perturbed Bayesian model, 
	and $q$ an arbitrary distribution that approximates the true posterior $p(\param|\data)$.
	The transformation $\tau(\param) = \param + h \nabla \log \tilde{p}(\data, \param)$ will induce a change of measure such that the directional derivative $\nabla_{\tau} \kl[q||p]$ is non-positive if:
	\begin{equation}
	\label{eq:the_inequality}
		\E_q[\nabla \log p^\top \nabla \log \tilde{p}] \ge -\E_q[\mathrm{tr}{\{\Hessian \log \tilde{p}\}}]
	\end{equation}
\end{theorem}
\begin{proof}
	The result is produced by a simple manipulation of \eqref{eq:kl_dir_deriv}, and by setting $\nabla_{\tau} \kl[q||p] < 0$.
\end{proof}

The inequality in \eqref{eq:the_inequality} is not always satisfied, as the Hessian can contain negative numbers in its diagonal; e.g., the second derivatives for $\param$ near local maxima should be negative.

As an example where the inequality in \eqref{eq:the_inequality} is violated, consider a convex \emph{unperturbed}
joint log-likelihood, i.e.\ different particles $\param\sim q$ optimize $\log p(\data,\param)$.
Eventually, the different gradients $\nabla\log p$ will approach zero for any $\param$. 
The directional derivative expectation will also approach zero, as all points converge to the same maximum.
The directional derivative of the \KL divergence will tend to be positive, implying that further application of the transformation $\tau$ results in poorer approximation of the true posterior.

In the general case, it is rather difficult to reason precisely about the value of $\nabla_{\tau} \kl[q||p]$.
Nevertheless, we conjecture that the introduction of a perturbation will make the inequality in \eqref{eq:the_inequality} less likely to be violated.
We demonstrate this effect for certain kinds of prior and likelihood in the following section.

\subsection{Gaussian prior and likelihood}

Let $f(\mathbf{x};\param)$ be the output of a nonlinear model (i.e.\ a \dnn) and  let $\param \in \mathbb{R}^m$ be a vectorization of its parameters including weight and bias terms.
We shall consider a Gaussian prior $\norm(0, \alpha^2 I_m)$ and a likelihood function of the form $\norm(f(\mathbf{x};\param), \sigma^2)$.

Let $\norm(\tilde{\param}, \alpha^2 I_m)$  denote a perturbed version of the prior, where $\tilde{\param} \sim \mathcal{N}(0, \alpha^2 I_m)$.
Let the perturbed version of the data be $\tilde{\data}$, where for all $(\mathbf{x}, y) \in \data$ and $(\mathbf{x}, \tilde{y}) \in \tilde{\data}$ we have $\tilde{y} = y + \tilde{y}_0$ and $\tilde{y}_0 \sim \mathcal{N}(0, \sigma^2)$.
The perturbed version of the log-likelihood will be:
\begin{equation}
\label{eq:logp_tilde}
\log \tilde{p}(\data,\param) = -\sum_{\mathbf{x}, \tilde{y} \in \tilde{\data}} \frac{(f(\mathbf{x}; \param) - \tilde{y})^2}{2\sigma^2} - \sum_{j=1}^m \frac{(\param_j - \tilde{\param}_j)^2}{2\alpha^2}
\end{equation}
For the gradient and the Hessian trace of the perturbed log-likelihood above, we have:
\begin{equation}
	\label{eq:E_grad}
	\nabla \log \tilde{p}(\data,\param) = \nabla \log p(\data,\param) + \sum_{\mathbf{x}, \tilde{y} \in \tilde{\data}} \frac{\tilde{y}_0}{\sigma^2} \nabla f(\mathbf{x};\param) + \frac{\tilde{\param}_j}{\alpha^2}
\end{equation}
\begin{equation}
	\label{eq:E_hess}
	\mathrm{tr}{\{\Hessian \log \tilde{p}(\data,\param)\}} = \mathrm{tr}{\{\Hessian \log p(\data,\param)\}}
	+ \sum_{\mathbf{x}, \tilde{y} \in \tilde{\data}} \left( \frac{\tilde{y}_0}{\sigma^2} \sum_{j=1}^m \partial_{\param_j^2} f(\mathbf{x};\param) \right)
\end{equation}

During the optimization process, each particle $\param$ is associated with a particular random perturbation.
We have not made any specific assumptions regarding the approximating distribution $q$, therefore the random variable $\param \sim q$ and the perturbations $\tilde{y}_0$ and $\tilde{\param}$ will be mutually independent.

We leverage this mutual independence and we exploit certain properties of the Gaussian assumptions, in order to further develop Eq.~\eqref{eq:the_inequality} into the theorem that follows.


\begin{theorem} 
	\label{th:gaussian_theorem}
	Let $\log \tilde{p}(\data,\param)$ be a perturbed Bayesian nonlinear model with prior $\norm(0, \alpha^2 I_m)$ and likelihood $\norm(f(\mathbf{x};\param), \sigma^2)$, with perturbations $\tilde{y} \sim \norm(f(\mathbf{x};\param), \sigma^2)$ and $\tilde{\param} \sim \norm(0, \alpha^2 I_m)$.
	Let $q$ be an arbitrary distribution that approximates the true posterior $p(\param|\data)$.
	The transformation $\tau(\param)$ will induce a change of measure such that the directional derivative $\nabla_{\tau} \kl[q||p]$ is non-positive if:
	\begin{equation}
	\label{eq:gaussian_theorem}
		E_{q, \tilde{y}_0, \tilde{\param}}\left[\Vert \nabla \log \tilde{p} \Vert_2^2\right]
		\ge E_q\left[\sum_{\mathbf{x}, y \in \data} \left(\frac{f(\mathbf{x};\param) - y}{\sigma^2} \sum_{j=1}^m \partial_{\param_j^2} f(\mathbf{x};\param) \right)\right]
	\end{equation}
\end{theorem}
\begin{proof}
	We can calculate the expectation of the \KL directional derivative w.r.t.\ $\tilde{y}_0$, $\tilde{\param}$ by noticing that $\E_{\tilde{y}_0} [\tilde{y}_0] = 0$ and $\E_{\tilde{\param}} [\tilde{\param}] = 0$:
	\begin{equation}
	\label{eq:E_kl_deriv}
	\E_{\tilde{y}_0, \tilde{\param}}\left[\nabla_{\tau} \kl[q||p]\right]
	= -\E_q[\nabla \log p^\top \nabla \log p  +  \mathrm{tr}{\{\Hessian \log p\}}]
	\end{equation}
	We next express the gradient norm $\Vert \nabla \log p \Vert_2^2$  in terms of the expectation $\E_{\tilde{y}_0, \tilde{\param}} \left[\Vert \nabla \log \tilde{p} \Vert_2^2\right]$ (see Lemma \ref{lemma_fortheorem2} in Appendix \ref{ssec:proof_nonlin_normal}).
	The same quadratic terms appear in both the norm expectation and the Hessian, thus the inequality simplifies to Eq.~\eqref{eq:gaussian_theorem}.
	See details in Section \ref{ssec:proof_nonlin_normal} of the supplement.
\end{proof}

Theorem \ref{th:gaussian_theorem} involves the expectation of the perturbed gradient norm. 
This should be a positive value that approaches zero, as the optimization converges.
The theorem demands that this positive value is larger in expectation than 
a summation involving the diagonal second derivatives of $f(\mathbf{x}; \param)$.

The fraction term in r.h.s.\ of \eqref{eq:gaussian_theorem} may have large absolute values if the particles $\param \sim q$ are far from the posterior, but so will the perturbed gradient norm.
However, the difference $f(\mathbf{x};\param)-y$ is not evaluated in absolute value; if the data are reasonably approximated, any discrepancies will be averaged out. 
Nevertheless, it is rather difficult to reason about the magnitude of this difference in the general case.



\vspace{-10pt}
\paragraph{Remarks on linear and piecewise-linear models}
The second-order derivative term in r.h.s.\ of \eqref{eq:gaussian_theorem} represents the curvature of the regression model learned.
This term can be further simplified for certain families of functions.
For a linear model, $f$ will be linear w.r.t.\ the parameters; this means that the directional derivative of \KL divergence is \emph{guaranteed} to be non-positive.
This result can be extended to functions that are only piecewise-linear w.r.t.\ their parameters.
A popular \dnn design choice in the recent years involves \relu activations, which are known to produce piecewise-linear models.
It can be easily seen that the Hessian of a \relu network is defined almost everywhere and its diagonal contains zeros. 
The set for which the Hessian is not defined has zero measure; for the same set the gradient is not defined either, but this has little effect on the usability of \relu models.

As a final remark for models for which $\partial_{\param_j^2} f(\mathbf{x};\param)=0$, the \KL divergence will only decrease over the course of optimization, until its directional derivative finally becomes zero.
That does not guarantee that the \KL divergence is optimized, as the derivative would have to be zero towards \emph{any} direction.
For example, the directional derivative could be zero if the direction of $\tau$ is orthogonal to the \KL divergence gradient; that is a very particular case however.
Nevertheless, for the linear case we know that the perturbed optimization produces exact samples from the posterior (Section \ref{sec:lineal_models}).



\paragraph{Example -- Toy \relu network}
\label{ssec:example_toy}
We demonstrate the effect of Theorem \ref{th:gaussian_theorem} through an example that allows us to monitor the progress of optimization.
We consider: $f(x) = \text{\relu}(w_1 x) + \text{\relu}(w_2 x)$, where $x \in \mathbb{R}$.
The model is parameterized only through the weights $w_1, w_2 \in \mathbb{R}$;
we can effectively calculate the divergence from the posterior distribution by means of \emph{kernel density estimation} (\KDE) and numerical integration\footnote{We have used the bivariate \KDE and integration routines available in the python packages \texttt{scipy.stats} and \texttt{scipy.integrate} correspondingly.}.
The  posterior is analytically intractable, so we resort to \mcmc \citep{Gelman1992a} to approximate the ground truth.
\nocite{Gelman1992a}
Figure \ref{fig:toy_relu} summarizes the results on a synthetic dataset of $24$ points.
\begin{figure}
	\includegraphics[width=\linewidth]{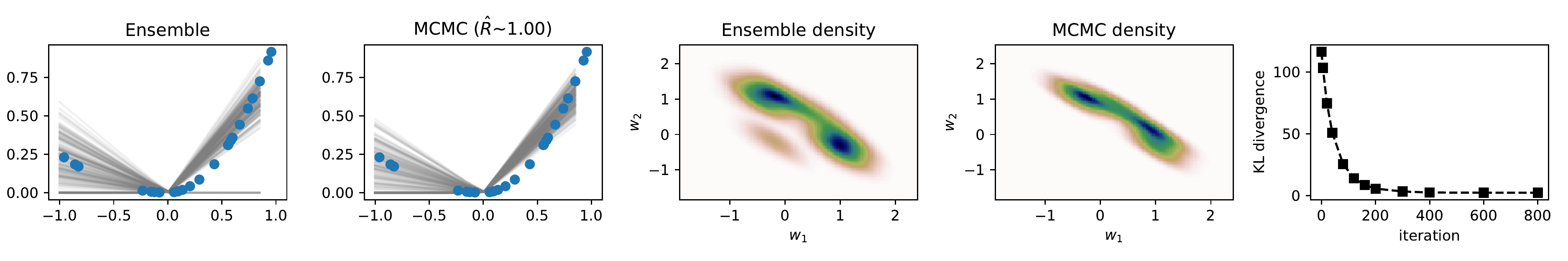}
	\vspace*{-20pt}
	\caption{Toy \relu network -- First two plots: Predictive models by bootstrap ensembles (left) and \mcmc (right). Next two plots: Densities of the bivariate weight distributions (via \KDE). Last plot: Evolution of KL divergence between the ensemble-based scheme and \mcmc.}
	\label{fig:toy_relu}
	\vspace*{-10pt}
\end{figure}
Samples from the predictive model distributions can be seen in the first two plots, ensemble-based on the left and \mcmc on the right; details on the setup can be found in Appendix \ref{ssec:toy_example}.
In the next two plots, we see the respective estimated densities (via \KDE) for the bivariate weight distribution.
In the final plot of Figure \ref{fig:toy_relu}, we see the evolution of \kl-divergence from the true posterior as optimization progresses.
Although the ensemble method results in a distribution that is not exactly the same as \mcmc, this example shows how Theorem \ref{th:gaussian_theorem} works.
The evolution of \kl-values forms a strictly decreasing curve, which is in agreement with our theory: any optimization step on an ensemble of \relu networks should have a non-detrimental effect on approximation quality.


\section{Conclusions}
Ensemble learning has found applications in problems where quantification of uncertainty matters, particularly those involving \dnns \citep{Lakshminarayanan2017,Osband2018,pearce2018_icml}.
It offers a practical means to uncertainty quantification by relying exclusively on optimization.
While this has been shown to work well empirically, there has been no attempt to formally characterize the underpinning theory beyond the linear case. 

In this paper, we employed variational arguments to establish a connection between a certain kind of ensemble learning and Bayesian inference beyond linear regression with a Gaussian likelihood, for which this connection was already known. 
In particular, we interpreted the particles associated to perturbed versions of the joint log-likelihood as samples from a distribution approximating the posterior over model parameters. 
We then derived conditions under which any optimization steps of these particles yields an improvement of the divergence between the approximate and the actual posterior.
Remarkably, the conditions do not require any particular form for the approximation and they are purely geometrical, involving first and second derivatives of the modeled function with respect to model parameters, giving insights on the behavior of ensemble learning on a number of interesting models. 
In particular, we showed how these results can be applied to \dnns with \relu activations to establish that the optimization of the particles 
is guaranteed to have no detrimental contribution to the \kl divergence. 

We believe that this work makes a significant step in broadening the scope of ensemble methods for quantification of uncertainty.
However, we are considering further extensions of our analysis to include more general classes of models, likelihoods, and priors, as well as investigations on how to use the criterior in Theorem~1 to derive novel particle-based variational inference methods.

\vspace{100pt}
\appendix

\section{Results for Bayesian linear models}
\label{sec:suppl_linear}

Consider the following Bayesian liner model with $\param = \wvec \in \mathbb{R}^m$:
\begin{equation}
y = \wvec^\top \phi(\mathbf{x}) + \epsilon
\end{equation}
with prior $\wvec \sim \mathcal{N}(0, \alpha^2 I_m)$ and noise model $\epsilon \sim \mathcal{N}(0, \sigma^2)$.
If the labels are perturbed according to the likelihood so that $\epsilon \sim \mathcal{N}(0, \sigma^2 I)$, and the regularisation term is a sample from the prior so that $\tilde{\wvec} \sim \mathcal{N}(0, \alpha^2 I_m)$,
then the MAP solution will be:
\begin{equation}
\wvec_* = \frac{1}{\sigma^2} A^{-1} \Phi (\mathbf{y} + \epsilon)
+ \frac{1}{\alpha^2} A^{-1} \tilde{\wvec}
\end{equation}
Regarding the covariance of $\wvec_*$ we have:
\begin{align*}
\E_{\epsilon,\tilde{\wvec}}[(\wvec_* - \bar{\wvec}) (\wvec_* - \bar{\wvec})^\top] 
&= \E_{\epsilon} \left[\frac{1}{\sigma^2} A^{-1} \Phi \epsilon \epsilon^\top \Phi^\top A^{-1} \frac{1}{\sigma^2} \right]
+ \E_{\tilde{\wvec}} \left[\frac{1}{\alpha^2} A^{-1} \tilde{\wvec} \tilde{\wvec}^\top A^{-1} \frac{1}{\alpha^2} \right] \\
&= A^{-1}  \frac{1}{\sigma^2} \Phi \Phi^\top  A^{-1}
+ A^{-1}  \frac{1}{\alpha^2} A^{-1} \\
&= A^{-1} \left(\frac{1}{\sigma^2} \Phi \Phi^\top + \frac{1}{\alpha^2} \right) A^{-1}\\
&= A^{-1}
\end{align*}
which is exactly the correct posterior weight covariance, since $A = \frac{1}{\sigma^2} \Phi \Phi^\top + \frac{1}{\alpha^2} I_m$.

\vspace{50pt}

\section{Detailed derivation for non-linear models}

In a gradient ascent scheme, the update will be described by the following transformation:
\begin{equation*}
\tau(\param) = \param + h \nabla \log \tilde{p}(\data, \param)
\end{equation*}
The transformed density after applying the differentiable transform $\tau^{-1}$ will be:
\begin{equation*}
\log p_{\tau^{-1}}(\param|\data) = \log p(\tau(\param)|\data) + \log\det \mathbf{J}_{\tau(\param)}
\end{equation*}
where by $\mathbf{J}_{\tau(\param)}$ we denote the Jacobian of $\tau(\param)$.
Notice that we have the Jacobian of the gradient vector, for which we have:
\begin{equation*}
\mathbf{J}_{\tau(\param)} = \mathbf{J}_\param + h \nabla^2 \log \tilde{p}(\data, \param)
= I_m + h \Hessian \log \tilde{p}(\data, \param)
\end{equation*}
Finally, we obtain:
\begin{equation*}
\log p_{\tau^{-1}}(\param|\data) = \log p(\tau(\param)|\data) + \log\det\{I_m + h \Hessian \log \tilde{p}(\data, \param)\}
\end{equation*}

Given that $\delta_h = \kl[q||p_{\tau^{-1}}] - \kl[q||p]$, we obtain the directional derivative by considering the limit: 
\begin{equation*}
\begin{split}
\lim_{h \to 0} \frac{\delta_h}{h}
&= -\E_q \left[
\lim_{h\to 0} \frac{
	\log p(\tau(\param) | \data) + \log\det\{I_m + h \Hessian \log \tilde{p}(\data,\param)\}
	- \log p(\param|\data)
}{h}
\right]
\\
&= -\E_q \left[
\lim_{h\to 0} \frac{
	\log p(\data, \tau(\param))
	+ \log \det\{I_m + h \Hessian \log \tilde{p}(\data,\param)\}
	- \log p(\param,\data)
}{h}
\right]
\\
&= -\E_q \left[
\lim_{h\to 0} \frac{\log p(\data, \param + h \nabla \log \tilde{p}) - \log p}{h}
+
\lim_{h\to 0} \frac{ \log\det\{I_m + h \Hessian \log \tilde{p}\} }{h}
\right]
\end{split}
\end{equation*}
Where we refer to the log-densities $\log p(\data, \param)$ and $\log \tilde{p}(\data, \param)$ simply as $\log p$ and $\log \tilde{p}$.


\vspace{50pt}

\section{Results for non-linear models with Gaussian likelihood and prior}

Let $\log p(\data,\param)$ be the joint log-likelihood of a model with likelihood $\norm(f(\mathbf{x};\param), \sigma^2)$ and prior $\norm(0, \alpha^2 I_m)$.
We consider a perturbation of the prior $\tilde{\param} \sim \mathcal{N}(0, \alpha^2 I_m)$, and the data 
$\tilde{\data}$ such that for all $(\mathbf{x}, y) \in \data$ and $(\mathbf{x}, \tilde{y}) \in \tilde{\data}$ we have $\tilde{y} = y + \tilde{y}_0$, where $\tilde{y}_0 \sim \mathcal{N}(0, \sigma^2)$.
Then the perturbed version of the log-likelihood will be:
\begin{equation}
\label{eq:logp_tilde_suppl}
\log \tilde{p}(\data,\param) = -\sum_{\mathbf{x}, \tilde{y} \in \tilde{\data}} \frac{(f(\mathbf{x}; \param) - \tilde{y})^2}{2\sigma^2} - \sum_{j=1}^m \frac{(\param_j - \tilde{\param}_j)^2}{2\alpha^2}
\end{equation}

\vspace{20pt}

\subsection{Gradient and Hessian of log-likelihood}

For the components of the gradient we have:
\begin{equation}
\label{eq:logp_grad}
\frac{\partial}{\partial\param_j} \log p(\data, \param)
= - \sum_{\mathbf{x}, y \in \data} \frac{f(\mathbf{x};\param) - y}{\sigma^2} \partial_{\param_j} f(\mathbf{x};\param) - \frac{\param_j}{\alpha^2}
\end{equation}


For the trace of the Hessian, we simply need the derivatives of the gradient components:
\begin{equation}
\label{eq:logp_hess}
\begin{split}
\mathrm{tr}{\{\Hessian \log p(\data, \param)\}}
& = \sum_{j=1}^m \frac{\partial^2}{\partial\param_j^2} \log p(\data, \param) \\
& = \sum_{j=1}^m \left( - \sum_{\mathbf{x}, y \in \data} \frac{1}{\sigma^2} (\partial_{\param_j} f(\mathbf{x};\param))^2 
- \sum_{\mathbf{x}, y \in \data} \frac{f(\mathbf{x};\param) - y}{\sigma^2} \partial_{\param_j^2} f(\mathbf{x};\param)
-\frac{1}{\alpha^2} \right) \\
& = -\frac{1}{\sigma^2} \sum_{\mathbf{x}, y \in \data} \Vert\nabla f(\mathbf{x};\param)\Vert_2^2
-\frac{m}{\alpha^2} 
- \sum_{\mathbf{x}, y \in \data} \left(\frac{f(\mathbf{x};\param) - y}{\sigma^2}
\sum_{j=1}^m \partial_{\param_j^2} f(\mathbf{x};\param) \right)
\end{split}
\end{equation}

\vspace{20pt}

\subsection{Expectation of perturbed gradient}
\label{ssec:suppl_E_glogp_tilde}
The components of the perturbed gradient will be:
\begin{equation}
\begin{split}
\label{eq:perturb_grad}
\frac{\partial}{\partial\param_j} \log \tilde{p}(\data, \param)
&= -  \sum_{\mathbf{x}, \tilde{y} \in \tilde{\data}} \frac{f(\mathbf{x};\param) - \tilde{y}}{\sigma^2} \partial_{\param_j} f(\mathbf{x};\param) - \frac{\param_j - \tilde{\param}_j}{\alpha^2}\\
&= \frac{\partial}{\partial\param_j} \log p(\data, \param)
+ \sum_{\mathbf{x}, \tilde{y} \in \tilde{\data}} \frac{\tilde{y}_0}{\sigma^2} \partial_{\param_j} f(\mathbf{x};\param) + \frac{\tilde{\param}_j}{\alpha^2}
\end{split}
\end{equation}
It is easy to see that for the gradient of a perturbed log-likelihood, its expectation with respect to the perturbation will be equal to the unperturbed gradient:
\begin{equation}
\E_{\tilde{y}_0, \tilde{\param}} [\nabla \log \tilde{p}(\data, \param)] = \nabla \log p(\data, \param)
\end{equation}

\vspace{20pt}

\subsection{Expectation of perturbed Hessian}
\label{ssec:suppl_E_Hlogp_tilde}

For the trace we only need the diagonal components of the perturbed Hessian; by differentiating Equation \eqref{eq:perturb_grad} by $\param_j$ we get:
\begin{equation}
\begin{split}
\mathrm{tr}{\{\Hessian \log \tilde{p}(\data, \param)\}}
& = \sum_{j=1}^m  \frac{\partial^2}{\partial\param_j^2} \log p(\data, \param)
+ \sum_{\mathbf{x}, \tilde{y} \in \tilde{\data}} \left( \frac{\tilde{y}_0}{\sigma^2} \sum_{j=1}^m \partial_{\param_j^2} f(\mathbf{x};\param) \right)
\end{split}
\end{equation}
Since $\E_{\tilde{y}_0} [\tilde{y}_0] = 0$, we have:
\begin{equation}
\E_{\tilde{y}_0, \tilde{\param}} [\mathrm{tr}{\{\Hessian \log \tilde{p}(\data, \param)\}}] = \mathrm{tr}{\{\Hessian \log p(\data, \param)\}}
\end{equation}

\vspace{20pt}

\subsection{Proof of Theorem 2}
\label{ssec:proof_nonlin_normal}

Before proving Theorem 2, we shall review the following lemma:
\begin{lemma}
	\label{lemma_fortheorem2}
	Let $\log \tilde{p}(\data,\param)$ be a perturbed Bayesian non-linear model as in \eqref{eq:logp_tilde_suppl}.
	For the perturbation distributions we assume: $\tilde{y}_0 \sim \norm(0, \sigma^2)$ and $\tilde{\param} \sim \norm(0, \alpha I_m)$.
	Then, for arbitrary $\param$ we have:
	\begin{equation}
	\Vert \nabla \log p \Vert_2^2 = 
	\E_{\tilde{y}_0, \tilde{\param}} \left[ \Vert \nabla \log \tilde{p} \Vert_2^2	\right] 
	+ \frac{1}{\sigma^2} \sum_{\mathbf{x}, \tilde{y} \in \tilde{\data}} \Vert\nabla f(\mathbf{x};\param)\Vert_2^2 + \frac{m}{\alpha^2}
	\end{equation}
\end{lemma}
\begin{proof}
	From \eqref{eq:perturb_grad}, we have the following for the original gradient:
	\begin{equation}
	\nabla \log p(\data, \param) = - \sum_{\mathbf{x}, \tilde{y} \in \tilde{\data}} \frac{\tilde{y}_0}{\sigma^2} \nabla f(\mathbf{x};\param) - \frac{\tilde{\param}}{\alpha^2} + \nabla \log \tilde{p}(\data, \param)
	\end{equation}
	
	We consider the following joint expectation with respect to $\tilde{y}_0$ and $\tilde{\param}$:
	\begin{equation}
	\begin{split}
	\E_{\tilde{y}_0, \tilde{\param}} & [\nabla \log p(\data, \param)^\top \nabla \log p(\data, \param)] \\
	& = \E_{\tilde{y}_0, \tilde{\param}} 
	\left[ 
	\left(- \sum_{\mathbf{x}, \tilde{y} \in \tilde{\data}} \frac{\tilde{y}_0}{\sigma^2} \nabla f(\mathbf{x};\param) - \frac{\tilde{\param}}{\alpha^2} + \nabla \log \tilde{p}(\data, \param) \right)^\top 
	\left(- \sum_{\mathbf{x}, \tilde{y} \in \tilde{\data}} \frac{\tilde{y}_0}{\sigma^2} \nabla f(\mathbf{x};\param) - \frac{\tilde{\param}}{\alpha^2} + \nabla \log \tilde{p}(\data, \param) \right)
	\right] 
	\\
	&= \E_{\tilde{y}_0, \tilde{\param}} 
	\left[ 
	\sum_{\mathbf{x}, \tilde{y} \in \tilde{\data}} \frac{\tilde{y}_0^2}{\sigma^4} \Vert \nabla f(\mathbf{x};\param) \Vert_2^2
	+ \sum_{j=1}^m \frac{\tilde{\param}_j^2}{\alpha^4}
	+ \Vert \nabla \log \tilde{p} \Vert_2^2
	\right] 
	\end{split}
	\end{equation}
	Note that the terms of the polynomial that we have omitted are zero in expectation, because we have $\E[\tilde{y}_0]=0$ and $\E[\tilde{\param}]=0$.
	Also since we have $\E[\tilde{y}_0^2]=\sigma^2$ and $\E[\tilde{\param}_j]=\alpha^2$, the expectation becomes:
	\begin{equation}
	\nabla \log p(\data, \param)^\top \nabla \log p(\data, \param) =
	\frac{1}{\sigma^2} \sum_{\mathbf{x}, \tilde{y} \in \tilde{\data}} \Vert\nabla f(\mathbf{x};\param)\Vert_2^2	+ \frac{m}{\alpha^2} 
	+ \E_{\tilde{y}_0, \tilde{\param}} \left[ \Vert \nabla \log \tilde{p} \Vert_2^2	\right]
	\end{equation}
	Now we can move to the main theorem.
\end{proof}

\vspace{20pt}
\begin{customthm}{2}
	Let $\log \tilde{p}(\data,\param)$ be a perturbed Bayesian non-linear model with prior $\norm(0, \alpha^2 I_m)$ and likelihood $\norm(f(\mathbf{x};\param), \sigma^2)$, with perturbations $\tilde{y}_0 \sim \norm(f(\mathbf{x};\param), \sigma^2)$ and $\tilde{\param} \sim \norm(0, \alpha^2 I_m)$.
	Let $q$ be an arbitrary distribution that approximates the true posterior $p(\param|\data)$.
	The transformation $\tau(\param)$ will induce a change of measure such that the directional derivative $\nabla_{\tau} \kl[q||p]$ is non-positive if:
	\begin{equation*}
	\begin{split}
	E_{q, \tilde{y}_0, \tilde{\param}}\left[\Vert \nabla \log \tilde{p} \Vert_2^2\right]
	\ge E_q\left[\sum_{\mathbf{x}, y \in \data} \left(\frac{f(\mathbf{x};\param) - y}{\sigma^2} \sum_{j=1}^m \partial_{\param_j^2} f(\mathbf{x};\param) \right)\right]
	\end{split}
	\end{equation*}
\end{customthm}


\begin{proof}
	
	According to Theorem 1, the directional derivative $\nabla_{\tau} \kl[q||p]$ is non-positive if the following holds:
	\begin{equation*}
	\E_q[\nabla \log p^\top \nabla \log \tilde{p}] \ge -\E_q[\mathrm{tr}{\{\Hessian \log \tilde{p}\}}]
	\end{equation*}

	For a non-linear model with Gaussian prior and likelihood, the gradient and the Hessian trace of the perturbed log-likelihood $\log \tilde{p}$ will have expectations:
	\begin{align}
	\E_{\tilde{y}_0, \tilde{\param}} [\nabla \log \tilde{p}] &= \nabla \log p \\
	\E_{\tilde{y}_0, \tilde{\param}} [\mathrm{tr}{\{\Hessian \log \tilde{p}\}}] &= \mathrm{tr}{\{\Hessian \log p\}}
	\end{align}
	Derivations for the expectations above can be found in Sections \ref{ssec:suppl_E_glogp_tilde} and \ref{ssec:suppl_E_Hlogp_tilde} of the supplementary material.
	Also, we can expand $\nabla\log\tilde{p}$ in the following inner product using \eqref{eq:perturb_grad}:
	\begin{equation*}
	\begin{split}
	\nabla \log p^\top \nabla \log \tilde{p} &= \nabla \log p^\top \left( \nabla \log p
	+ \sum_{\mathbf{x}, \tilde{y} \in \tilde{\data}} \frac{\tilde{y}_0}{\sigma^2} \nabla f(\mathbf{x};\param) + \frac{\tilde{\param}}{\alpha^2}
	\right) 
	\\
	\E_{\tilde{y}_0, \tilde{\param}}[\nabla \log p^\top \nabla \log \tilde{p}] &= \nabla \log p^\top \nabla \log p
	\end{split}
	\end{equation*}
	If we consider the joint expectation with respect to $\param \sim q$, $\tilde{y}_0$ and $\tilde{\param}$, the condition specified by Theorem 1 can be approximated as follows:
	\begin{equation}
	\label{eq:simplified_theorem1}
	\begin{split}
	\E_{q, \tilde{y}_0, \tilde{\param}}[\nabla \log p^\top \nabla \log \tilde{p}] &\ge -\E_{q, \tilde{y}_0, \tilde{\param}}[\mathrm{tr}{\{\Hessian \log \tilde{p}\}}] \\
	\E_q[\nabla \log p^\top \nabla \log p] &\ge -\E_q[\mathrm{tr}{\{\Hessian \log p\}}]
	\end{split}
	\end{equation}


	Finally, if we use Lemma \ref{lemma_fortheorem2} on Equation \eqref{eq:simplified_theorem1} and we expand the Hessian according to Equation \eqref{eq:logp_hess}, we obtain:
	\begin{equation*}
	\begin{split}
	&\E_q\left[ 
	\frac{1}{\sigma^2} \sum_{\mathbf{x}, \tilde{y} \in \tilde{\data}} \Vert\nabla f(\mathbf{x};\param)\Vert_2^2
	+ \frac{m}{\alpha^2} + \E_{\tilde{y}_0, \tilde{\param}} \left[ \Vert \nabla \log \tilde{p} \Vert_2^2	\right]
	\right] \\
	& \ge -\E_q\left[
	-\frac{1}{\sigma^2} \sum_{\mathbf{x}, y \in \data} \Vert\nabla f(\mathbf{x};\param)\Vert_2^2
	-\frac{m}{\alpha^2} 
	- \sum_{\mathbf{x}, y \in \data} \left(\frac{f(\mathbf{x};\param) - y}{\sigma^2}
	\sum_{j=1}^m \partial_{\param_j^2} f(\mathbf{x};\param) \right)
	\right]
	\end{split}	
	\end{equation*}
	which easily simplifies to the condition required for non-positive $\nabla_{\tau} \kl[q||p]$ in Theorem 2.
	
\end{proof}

\vspace{50pt}

\section{Details on the experimental results}
\label{sec:additional}


We show experimentally that the aforementioned ensemble-based scheme can perform competitively in regression tasks, in spite of its simple nature.
As evaluation metrics, we use the test \emph{root mean squared error} (\rmse) and the test \emph{mean negative log-likelihood} (\mnll).

We apply ensemble-based regression using 200 particles on a \relu neural network with 1 and 4 hidden layers having 50 nodes each.
We consider a Gaussian likelihood model $\norm(0, \sigma^2)$, and a Gaussian prior on all weights $\norm(0, \alpha_w^2)$ and biases $\norm(0, \alpha_b^2)$.
The likelihood variance $\sigma^2$ has been selected for each model individually by performing grid search on validation \mnll, for which we withhold 20\% of the training examples.
The prior parameters have been fixed as $\alpha_w^2=1$ and $\alpha_b^2=1$.
We use \lbfgs for the optimization task (32 steps, step size: 0.5).
All datasets have been normalized so that the data (i.e.\ training inputs and labels) has zero mean and unit variance; the metrics reported are calculated after denormalization of the labels.

We compare our ensemble scheme against Stochastic-gradient Hamiltonian Monte Carlo (\sghmc) \cite{Chen14}, which samples from the true posterior distribution (1-layer only).
We have considered the same prior and likelihood as the one used for the ensemble method; the variance parameter of the likelihood has been chosen by cross-validation on \mnll, for which we withhold 20\% of the training examples.
In all cases, We have considered the learning rate to be equal to $0.001$ and the decay parameter was fixed to $0.01$.
We have adopted  the scale-adapted version in \cite{BOHamiANN}, where the rest of the parameters of \sghmc are adjusted during a burn-in phase.
In the experiments considered, we have used a burn-in of $2000$ steps, and we have kept $1$ sample for every $2000$ steps.
We have performed $10$ restarts; for each restart the trajectory was initialized by a different sample from the prior distribution.
Finally, we generated $200$ samples in total, as many as we used for the ensemble method.

As another comparison baseline we use Monte Carlo dropout (\mcd) \citep{Gal16} featuring a network of identical structure as in our setup.
The hyper-parameters have been set according to the guidelines of \citet{gal2018w} using the code of the authors available on-line\footnote{https://github.com/yaringal/DropoutUncertaintyExps}.

We also include in the comparison \gp regression \citep{Rasmussen06}; we have used the algorithms available in the \texttt{GPFlow} library \citep{GPflow17}.
We use the exact \gp algorithm in all cases except from the ``Protein'' dataset, for which we employ sparse variational \gp regression \citep{Titsias09} with 400 inducing points.
We have used an isotropic squared-exponential kernel; its hyper-parameters have been selected by optimizing the marginal log-likelihood (or \elbo, if the variational method is used).

The datasets and results are summarized in Tables \ref{tab:rmse} and \ref{tab:mnll}.
Experiments have been repeated for 10 random training/test splits, except for the ``Protein'' dataset, for which we have used 5 splits.
We report the average values for test \rmse and \mnll, plus/minus one standard deviation.
Results that have been significantly better have been marked in bold.
We see that the ensemble-based approach produces results that are either competitive to other methods or better.
In the main paper we have included only 4-layer results for \mcd, as this has given the best performance for \mcd in most cases.

The simple ensemble-based strategy presented in the previous section has been an adequate tool not only to perform robust regression, but also to quantify predictive uncertainty.

\begin{table*}[ht]
\caption{Regression results: Average test RMSE} \label{tab:rmse}
\begin{center}
{\footnotesize 
\begin{tabular}{lr||cc|cc|c|c}
                 &              &\multicolumn{2}{c|}{\textbf{Ensemble}}       &\multicolumn{2}{c|}{\textbf{MC-Dropout}}  &\textbf{SGHMC}    &\textbf{GP}\\
\textbf{DATASET} &\textbf{SIZE} &1-layer               &4-layer               &1-layer               &4-layer            &1-layer           &            \\
\hline \\
Boston           &506           &3.17$\pm$0.65         &3.17$\pm$0.60         &2.96$\pm$0.40         &2.93$\pm$0.27      &3.55$\pm$0.57     &3.18$\pm$0.63 \\
Concrete         &1030          &5.17$\pm$0.33         &5.47$\pm$0.75         &4.86$\pm$0.26         &4.74$\pm$0.34      &6.17$\pm$0.40     &5.58$\pm$0.35 \\
Energy           &768           &0.45$\pm$0.04         &0.65$\pm$0.38         &0.52$\pm$0.05         &0.45$\pm$0.04      &0.46$\pm$0.04     &0.48$\pm$0.05 \\
Kin8nm           &8192          &0.07$\pm$0.00         &0.06$\pm$0.00         &0.07$\pm$0.00         &0.08$\pm$0.00      &0.08$\pm$0.00     &0.07$\pm$0.00 \\
Naval            &11934         &0.00$\pm$0.00         &0.00$\pm$0.00         &0.00$\pm$0.00         &0.00$\pm$0.00      &0.00$\pm$0.00     &0.00$\pm$0.00 \\
Power            &9568          &4.12$\pm$0.11         &4.03$\pm$0.11         &3.99$\pm$0.12         &3.55$\pm$0.14      &4.23$\pm$0.11     &3.90$\pm$0.11 \\
Protein          &45730         &1.89$\pm$0.03         &1.89$\pm$0.03         &4.23$\pm$0.04         &3.47$\pm$0.02      &1.96$\pm$0.03     &1.88$\pm$0.03 \\
Wine-red         &1599          &0.61$\pm$0.02         &0.61$\pm$0.02         &0.61$\pm$0.02         &0.60$\pm$0.02      &0.62$\pm$0.02     &0.61$\pm$0.02 \\
Yacht            &308           &0.75$\pm$0.23         &0.76$\pm$0.35         &0.80$\pm$0.20         &1.49$\pm$0.28      &0.57$\pm$0.20     &0.49$\pm$0.18 \\
\end{tabular}
}
\end{center}
\end{table*}

\begin{table*}[ht]
\caption{Regression results: Average test MNLL} \label{tab:mnll}
\begin{center}
{\footnotesize 
\begin{tabular}{lr||cc|cc|c|c}
                 &              &\multicolumn{2}{c|}{\textbf{Ensemble}}       &\multicolumn{2}{c|}{\textbf{MC-Dropout}}  &\textbf{SGHMC}    &\textbf{GP}\\
\textbf{DATASET} &\textbf{SIZE} &1-layer               &4-layer               &1-layer               &4-layer            &1-layer           &            \\
\hline \\
Boston           &506           & 3.72$\pm$1.87        & 3.60$\pm$1.48        & 2.52$\pm$0.21        & 2.44$\pm$0.12     & 3.40$\pm$0.87    & 3.59$\pm$2.24 \\
Concrete         &1030          & 4.60$\pm$2.36        & 4.14$\pm$1.61        & 2.94$\pm$0.04        & 2.91$\pm$0.09     & 5.20$\pm$1.06    & 3.77$\pm$1.25 \\
Energy           &768           & 1.85$\pm$2.89        & 1.53$\pm$1.07        & 1.21$\pm$0.02        & 1.16$\pm$0.01     & 1.19$\pm$1.04    & 1.84$\pm$1.77 \\
Kin8nm           &8192          &-1.19$\pm$0.01        &-1.30$\pm$0.01        &-1.13$\pm$0.02        &-1.14$\pm$0.03     &-1.16$\pm$0.02    &-0.50$\pm$0.01 \\
Naval            &11934         &-5.61$\pm$0.03        &-5.45$\pm$0.12        &-4.44$\pm$0.01        &-4.48$\pm$0.00     &-4.54$\pm$0.28    &-5.98$\pm$0.00 \\
Power            &9568          & 2.88$\pm$0.03        & 2.98$\pm$0.06        & 2.80$\pm$0.02        & 2.63$\pm$0.03     & 2.90$\pm$0.04    & 6.92$\pm$2.94 \\
Protein          &45730         & 2.06$\pm$0.01        & 2.06$\pm$0.01        & 2.87$\pm$0.01        & 2.64$\pm$0.00     & 2.08$\pm$0.01    & 4.71$\pm$0.25 \\
Wine-red         &1599          & 0.89$\pm$0.04        & 0.87$\pm$0.04        & 0.92$\pm$0.03        & 0.89$\pm$0.03     & 0.94$\pm$0.04    & 1.07$\pm$0.01 \\
Yacht            &308           & 1.12$\pm$0.34        & 0.96$\pm$0.33        & 1.28$\pm$0.08        & 1.52$\pm$0.07     & 3.82$\pm$4.44    & 1.26$\pm$1.82 \\
\end{tabular}
}
\end{center}
\end{table*}

\vspace{50pt}
\section{Details on the examples}

\subsection{Toy \relu network}
\label{ssec:toy_example}

We consider the network: $f(x) = \text{\relu}(w_1 x) + \text{\relu}(w_2 x)$, where $x \in \mathbb{R}$ featuring a single hidden layer with 2 \relu nodes and fixed output layer.
The model is parameterized only through the weights $w_1, w_2 \in \mathbb{R}$.
We consider priors $w_i \sim \norm(0, 1)$, for $i \in \{1, 2\}$, and a Gaussian likelihood with variance $\sigma^2=0.05$.

In order to approximate the true posterior distribution, we have generated $200$ samples by using the Metropolis-Hastings algorithm \cite{Hastings70} featuring a Gaussian proposal with variance $0.5$.
For each one of $10$ restarts, we have discarded the first $10000$ samples, and then we have kept one sample every $5000$ steps.
As a convergence diagnostic, we report the average $\hat{R}$ statistic \cite{Gelman1992a} on the predictive models; 
the current choice of parameters has resulted in average $\hat{R}=1.00$, which is a strong indication of convergence.
In terms of the ensemble-based approach, we have generated $200$ particles from the prior distribution, on which we have applied $800$ steps of Algorithm 1 with $h=0.05$.

\subsection{Bayesian linear regression}
\label{ssec:linear_example}

We next consider the following 1-D linear regression model:
$f(x) = \wvec^\top \cos(\omega x-\pi/4)$,
where $\wvec \in \mathbb{R}^{D\times1}$ contains the weights of $D$ trigonometric features, $\omega \in \mathbb{R}^{D\times1}$ the corresponding frequencies, and the $\cos$ function is applied elementwise to its argument.
The values of $\omega$ are randomly initialized: $\omega \sim \norm(0, l^2 I_D)$, where $l$ is a user-defined lenghtscale (here: $l=1.6$ and $D=1024$).
Henceforth we consider $\omega$ to be fixed; the model is parametrized throught the linear parameter $\wvec$ only.

For likelihood $y|\wvec \sim \norm(f(x), 0.1)$ and prior $\wvec \sim \norm(0, I_D)$,
we have applied \adagrad ($400$ steps, $h=0.01$) on a population of $200$ particles for  a synthetic dataset of $64$ points.

\subsubsection*{A note on \kl divergence}

We consider the 1-D regression model of Section 5.2; we have $D$ trigonometric basis functions: $f(x) = \wvec^\top \cos(\omega x-\pi/4)$, where $\wvec \in \mathbb{R}^{D\times1}$ contains the weights of $D$ features and $\omega \in \mathbb{R}^{D\times1}$ is a vector of fixed frequencies.
For likelihood $y|\wvec \sim \norm(f(x), 0.1)$ and prior $\wvec \sim \norm(0, I_D)$, the true posterior over $\wvec$ is known to be Gaussian and it can be calculated analytically.

The first four plots of Figure \ref{fig:toy_rff_v0} depict the state of 10 particles at different optimization stages.
The black solid line represents the true posterior mean and the shaded area the 95\% of the true posterior support, which has been evaluated analytically.

The particle distribution is initialised as a Gaussian and it retains its Gaussian form over the course of optimization, as it undergoes linear transformations only.
Thus we can analytically estimate the KL divergence between the approximating distribution (by considering the empirical mean and covariance) and the true Gaussian posterior.
The rightmost subfigure in Figure \ref{fig:toy_rff_v0} depicts the evolution of the KL divergence (for 200 particles) from the true posterior; we see that curve is strictly decreasing, which is in line with the main argument of Theorem 2.

\begin{figure}[H]
	\includegraphics[width=\linewidth]{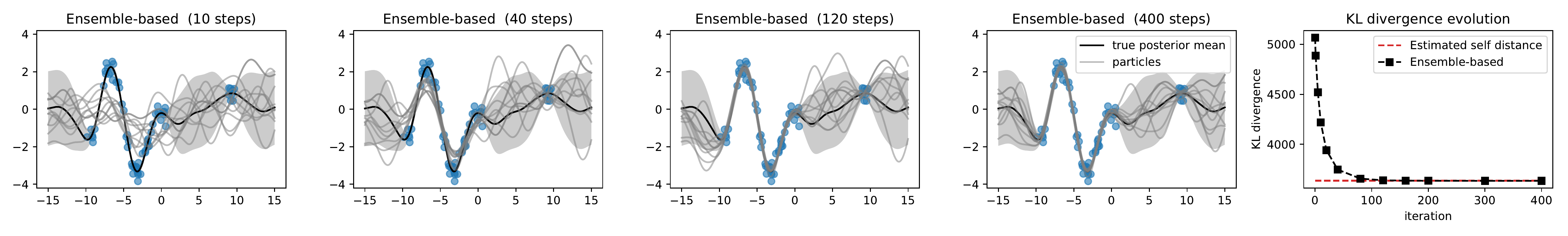}
	\vspace*{-15pt}
	\caption{Bayesian linear regression with trigonometric features -- State of 10 ensemble-based particles at different optimization stages. Rightmost figure: Evolution of \kl divergence (for 200 particles) from the true posterior.}
	\label{fig:toy_rff_v0}
\end{figure}

Notice that the \kl divergence decreases up to a level that is equal to the so-called \emph{self-distance},
which denotes the distance (\kl divergence in our case) of a sample from its actual generating distribution.
We know that $\kl[q||p] = 0$ iff $q=p$, but for a population of finite size, this value is expected to be larger than 0.
Regarding our example, the posterior is a $D$-variate Gaussian, so it has been easy to generate populations of size $200$, in order to estimate the self-distance.
The red dotted line denotes the average self-distance of 20 sampled populations of size $200$.
Convergence to the self-distance confirms that our ensemble strategy produces the true posterior for linear models, as discussed in Section 4.1 of the main paper.

\subsection{Regression \relu network}
\label{ssec:relu_example}

The final example is a 8-layer \dnn with 50 \relu nodes, featuring prior $\param \sim \norm(0, I_m)$, where $m=18000$, and likelihood $y|\param \sim \norm(f(x), 0.1)$.
As a point of reference for this example, we use samples of the Metropolis-Hastings algorithm featuring a Gaussian proposal with variance $0.01$.
We generated $200$ samples by performing $10$ restarts and having kept one sample every $20000$ steps, after discarding the first $40000$ samples.
Regarding convergence diagnostics, we have $\hat{R}=1.08$ on the predictive models.

\section{Extensions to classification}

Regarding the problem of classification, it is not obvious how to identify a data resampling strategy in the style of our approach to regression. 
We seek to create data replicates from a parametric distribution that reflects the properties of the Bernoulli likelihood.
Assume that a class label $y$ can take values in $\{0, 1\}$; if we locally fit a distribution $\mathrm{Bern}(p)$ to each $y$, the \ML parameter will be the one-sample mean i.e.\ $p_* = 0$ or $p_* = 1$.
The act of resampling from such a fitted model will deterministically produce either $0$ or $1$, eliminating thus any chance of perturbing the data.

Nevertheless, we think that classification can still benefit from bootstrapping.
Although an in-depth exploration of bootstrap-based classification is out of the scope of this work,
we demonstrate how a Gaussian-based perturbation can be extended to classification by employing local Gaussian approximations to the likelihood in the style of \cite{MiliosNIPS}.
In this way, we effectively turn classification into a regression problem in a latent space.
More specifically, each Bernoulli likelihood will be regarded as a degenerate beta distribution, to which we add a small regularisation term $\alpha_\epsilon = 0.01$.
In this way, the labels ``$0$'' will be represented as $\mathrm{Beta}(\alpha_\epsilon, 1+\alpha_\epsilon)$, while the labels ``$1$'' as $\mathrm{Beta}(1+\alpha_\epsilon, \alpha_\epsilon)$.
Classification is then treated as two parallel regression problems on the observed parameters of the beta distributions, for which we assume Gamma likelihoods\footnote{We leverage the fact that a beta-distributed random variable $x\sim \mathrm{Beta}(a,b)$ can be constructed as the ratio Gamma-distributed variables i.e.\ $x = x_a/(x_a+x_b)$, where $x_a \sim \mathrm{Gamma}(a, 1)$ and $x_b \sim \mathrm{Gamma}(b, 1)$.}.
The next step is to locally approximate the Gamma likelihood models with log-Normals by means of moment matching, which eventually translates to a regression problem with Gaussian likelihood in the log space.
Learning can be performed in this latent space and any predictions can be mapped back into probabilities by the soft-max function.

\begin{figure}[H]
	\centering
	\includegraphics[width=0.8\linewidth]{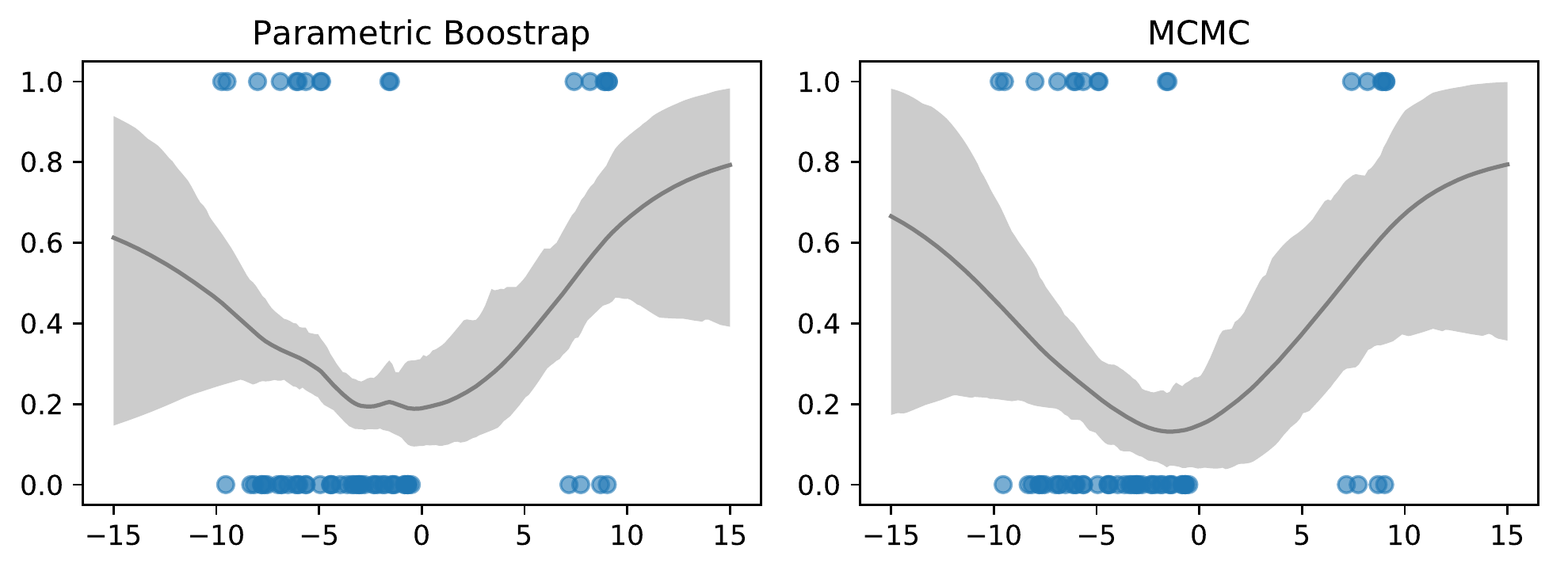}
	\caption{Application of Bootstrap to Classification. 
	}
	\label{fig:demo_classif}
\end{figure}

Figure \ref{fig:demo_classif} illustrates an example of applying parametric bootstrap to a synthetic classification problem, considering one hidden layer and \relu activations.
The shaded areas represent the 95\% support of the sampled distributions of classifiers.
The combination of bootstrap with locally approximated likelihoods on the left produces a fairly accurate approximation of the \mcmc result on the right.




\bibliographystyle{abbrvnat_nourl}
\bibliography{bibliography,filippone}


\end{document}